\documentclass[10pt,twocolumn,letterpaper]{article}

\usepackage{cvpr}
\usepackage{times}
\usepackage{epsfig}
\usepackage{graphicx}
\usepackage{amsmath}
\usepackage{amssymb}
\usepackage{amsthm}

\DeclareMathOperator*{\argmin}{arg\,min}
\usepackage{algorithm}
\usepackage{multirow}
\usepackage{multicol}
\usepackage[noend]{algpseudocode}
\usepackage{booktabs,caption,siunitx}

\newtheorem{theorem}{Theorem}

\cvprfinalcopy 

\def\cvprPaperID{6014} 

\ifcvprfinal\fi
\begin{document}

\title{Robustness of 3D Deep Learning in an Adversarial Setting}

\author{Matthew Wicker\\
University of Oxford\\
{\tt\small matthew.wicker@cs.ox.ac.uk}
\and
Marta Kwiatkowska\\
University of Oxford\\
{\tt\small marta.kwiatkowska@cs.ox.ac.uk}
}

\maketitle

\begin{abstract}
Understanding the spatial arrangement and nature of real-world objects is of paramount importance to many complex engineering tasks, including autonomous navigation. Deep learning has revolutionized state-of-the-art performance for tasks in 3D environments; however, relatively little is known about the robustness of these approaches in an adversarial setting. The lack of comprehensive analysis makes it difficult to justify deployment of 3D deep learning models in real-world, safety-critical applications. In this work, we develop an algorithm for analysis of 
pointwise robustness of neural networks that operate on 3D data. We show that current approaches presented for understanding the resilience of state-of-the-art models vastly overestimate their robustness. We then use our algorithm to evaluate an array of state-of-the-art models in order to demonstrate their vulnerability to occlusion attacks. We show that, in the worst case, these networks can be reduced to 0\% classification accuracy after the occlusion of at most 6.5\% of the occupied input space.
\end{abstract}

\section{Introduction}

Over the past several years, the machine learning community has worked to adapt the success of deep 2D vision algorithms to the 3D setting. Though initially slow to 
reach the performance
level of its 2D counterpart, recent advances have increased the accuracy of 3D deep learning pipelines by around 18\% on the ModelNet10 and ModelNet40 benchmarks \cite{ModelNet}. Now that 3D deep learning algorithms are able to achieve remarkable performance on standard benchmarks (currently topping out at 95\% accuracy), there have been many encouraging attempts to adapt these models to real-time, safety-critical scenarios such as landing zone detection for airborne drones \cite{dronelanding} and object recognition and classification for autonomous vehicles \cite{FrustrumPointNets, VoxelNet, multiview-driving}. In spite of the recent developments, the robustness of these pipelines remains poorly understood.

\begin{figure}[t]
\begin{center}
 \includegraphics[width=0.75\linewidth]{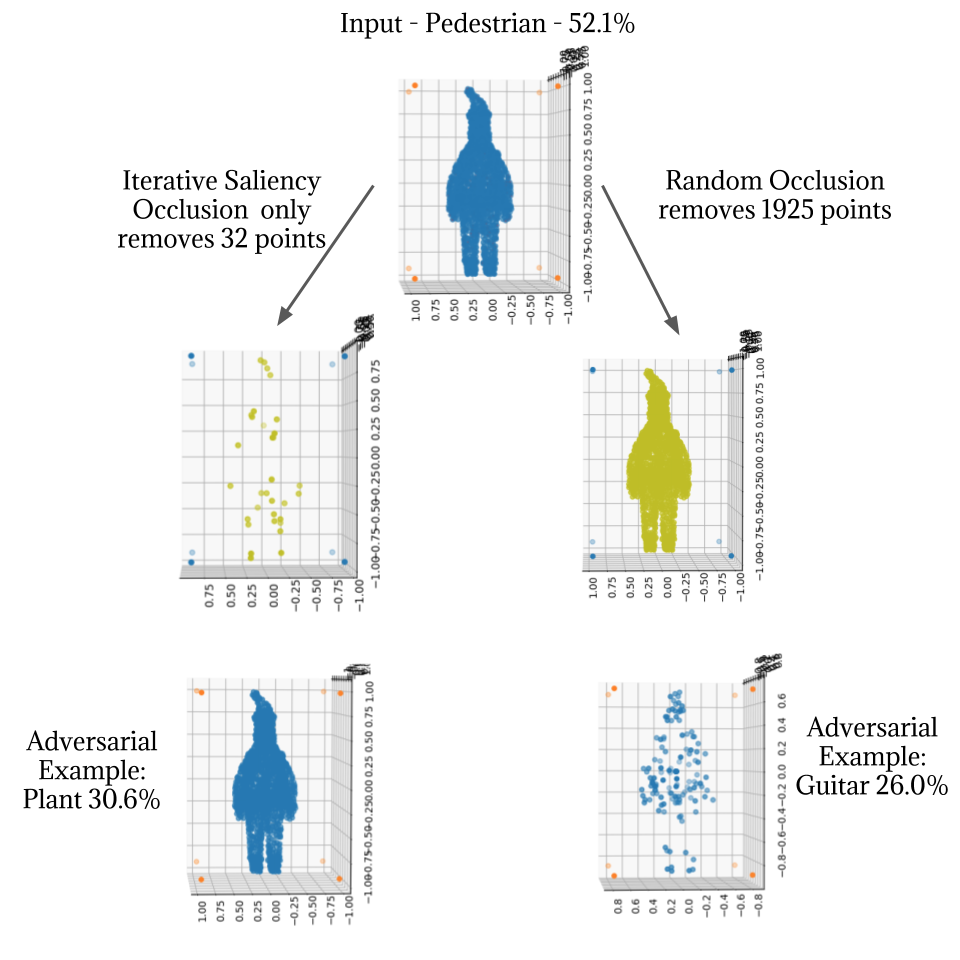}
\end{center}
\vspace*{-0.7cm}
   \caption{Despite the fact that PointNet \cite{PointNetOrig} is able to achieve high accuracy on the ModelNet40 test set, we show that by exploiting the low cardinality of the 
   induced critical point set we can cause the network to misclassify a man wearing a winter coat and beanie as a plant after only 32 out of 2048 points have been removed from the point cloud.}
\label{fig:pedestrian}
\end{figure}

\begin{figure*}
\begin{center}
\includegraphics[width=0.70\textwidth]{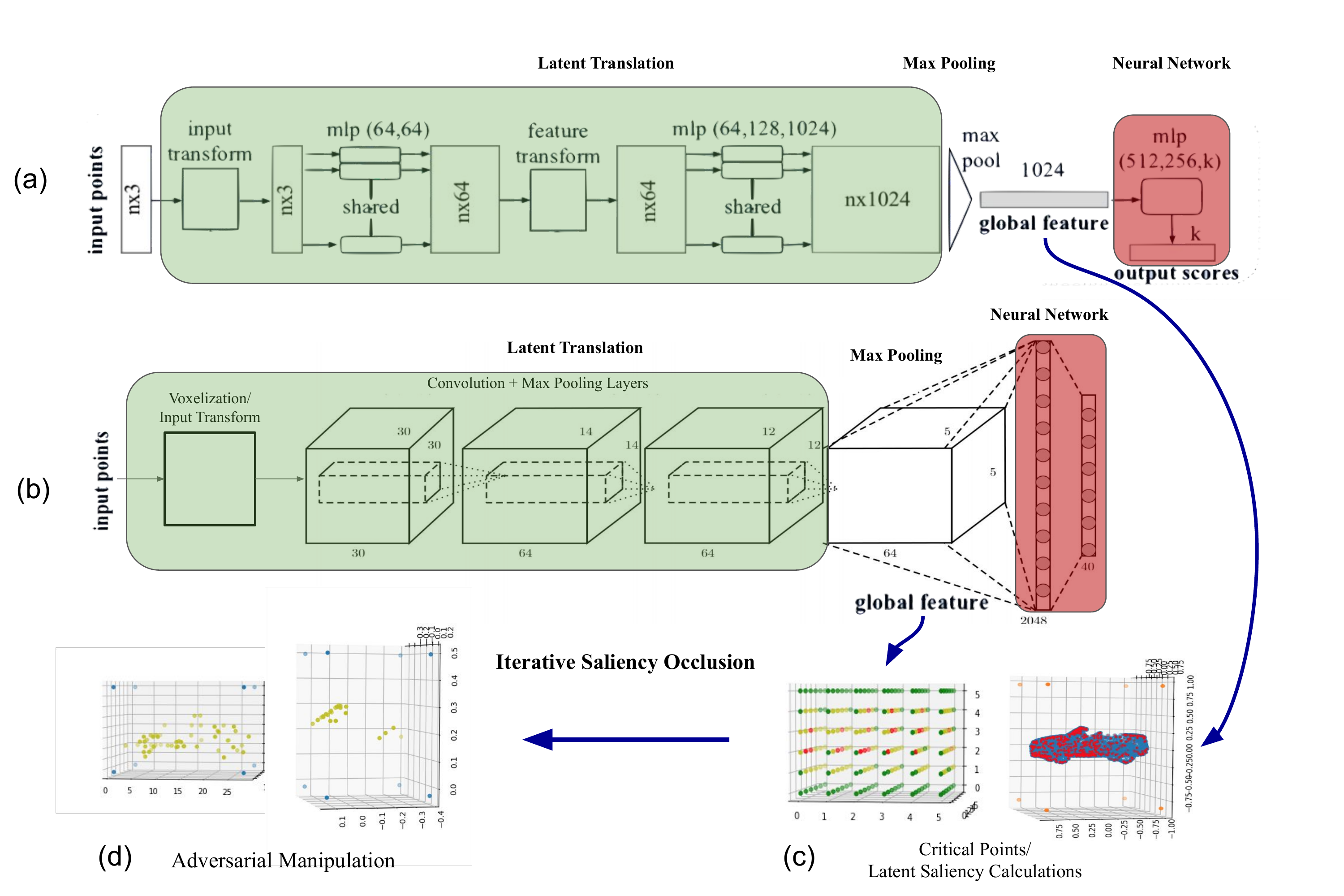}
\end{center}
\vspace*{-0.7cm}
   \caption{We outline a high-level unification of state-of-the-art 3D deep learning pipelines. (a) Figure from \cite{PointNetOrig} that demonstrates the PointNet architecture. (b) Figure from \cite{FusionNet} that demonstrates the architecture of a specific volumetric network. (c) Representations of salience pulled from point cloud and volumetric networks that are then used to generate (d) minimal adversarial manipulations. }
\label{fig:background}
\end{figure*}

The rapidly growing 
literature on the robustness (or lack thereof) of 2D vision algorithms casts doubt on the stability of their 3D relatives. In this work, we demonstrate the lack of robustness of 3D deep learning to adversarial occlusion, despite the results of random input occlusion suggesting that they are relatively invariant to perturbation. Developing comprehensive testing methods for these systems is of paramount importance given that misclassification of pedestrians (which is demonstrated in Figure \ref{fig:pedestrian}) as plants can induce poor planning by autonomous vehicles, 
an issue highlighted by a fatality
in the real world \cite{NYT-Uber}.

In the past few years, many methods for crafting adversarial examples have emerged \cite{propertiesOfNeuralNetworks, MadryPGD, JSMA, FGBBST}, including several frameworks for verification of safety \cite{reluplex, safetyVerification}. 
In contrast to studies of weaknesses
in deep learning models for image recognition due to their susceptibility to adversarial examples, very little work has been done to understand the robustness of 3D deep learning pipelines. This is in part because the data representation schemes utilized by 3D deep learning algorithms are not amenable to many current robustness analysis tools (see discussion in Section \ref{sec:safetyofdeeplearning}). By unifying the structure of foundational 3D deep learning pipelines, we solve this problem and lay the groundwork for future methods of safety testing and verification for these systems.

The value of understanding the robustness of 3D deep learning pipelines not only stems from their potential use in safety-critical systems (such as self-driving cars) but also the very noisy and unpredictable nature of collecting 3D data. Rarely will models receive full information about an object they are trying to recognize; rather, they will receive a single-angle, partially-occluded view of an object. In this work, we formalize a framework for moving towards an understanding of the performance of a wide spectrum of 3D deep learning algorithms. The efficient algorithm we propose could easily be used to further understand and improve the robustness and performance of future 3D deep learning algorithms. To this end, we offer several novel contributions to the study of 3D deep learning:

\begin{itemize}
  \item We present a unified view of volumetric and order-invariant 3D deep learning pipelines to exploit sensitivities to small changes in their inputs.
  \item We develop a novel algorithm to craft adversarial examples and provide guarantees about quality and existence of misclassifications.
  \item We use our algorithm to give a systematic occlusion analysis of the robustness of 3D deep learning algorithms that employ volumetric and point-cloud representations.\footnote{Code for all experiments in this paper can be found at https://github.com/matthewwicker/IterativeSalienceOcclusion}
\end{itemize}

We begin by briefly covering the the pertinent background for both 3D deep learning and the robustness analysis of deep learning algorithms. Following preliminaries, we formalize the problem of evaluating pointwise robustness under adversarial occlusion. After stating the problem, we give an algorithm that can operate in both the white-box and black-box settings. Further, we prove that this algorithm can give guarantees about the existence of adversarial examples. Finally, we use the algorithm to attack several state-of-the-art models in 3D deep learning. 

\section{Background}

In this section we aim to give an overview of the field of 3D deep learning as well as the state of the art of robustness analysis for deep learning algorithms.

\subsection{3D Deep Learning}\label{sec:3DDL}

The current renaissance of 3D deep learning methods can be attributed to both the wide availability of cheap sensors for collecting 3D data and the release of large standard datasets of 3D objects \cite{SydneyDataset, ModelNet, ALargeDatasetOfObjectScans}. Thanks to these datasets, 3D deep learning has enjoyed increased attention from machine learning practitioners. This has lead to an impressive leap in performance on standard benchmarks. Much of this progress can be attributed to novel data representation schemes, detailed below. 

\textbf{Volumetric Representations} One of the first methods for deep 3D shape classification called \textit{ShapeNet} \cite{ModelNet} achieved 77\% accuracy by representing data in a volumetric fashion. Volumetric representation of 3D shapes involves passing in a discretized 3D tensor (typically a cube), where the value of each entry in the 3D tensor represents the probability that an object inhabits that space. ShapeNet was surpassed by another network utilizing a volumetric approach to shape classification named \textit{VoxNet} \cite{VoxNet2015}, which, utilizing computationally expensive 3D convolutions, was able to achieve 83\% classification accuracy on the ModelNet benchmarks. 

Volumetric approaches have continued to find success outside the standard object recognition tasks and have been used in both landing site recognition for drones \cite{dronelanding} and in the classification of already localized objects in 3D driving scenes \cite{VoxelNet}.

\textbf{Multi-View Representations} \textit{Multi-View} networks take in a full 3D model of an object and from the model generate a series of 2D RGB images which are fed into 2D vision algorithms in order to arrive at a classification. Multi-view approaches in object classification \cite{MultiViewCNN2015} have remained consistently  state-of-the-art in terms of accuracy. However, the use of these networks in real-time scene and object recognition is non-trivial, and in some cases impossible, due to their inherent need for full 3D information of objects for classification. 

\textbf{Point-cloud representation} The recent seminal work by Qi et. al. \cite{PointNetOrig} (extended in \cite{Wang2018DGCNN} and \cite{Qi2017PointNetPlus}) uses neural networks to learn a point-set function that directly takes inputs from sensors (point clouds) and is able to classify them without the need for expensive operations such as conversion to more inflated domains (as is the case with volumetric and multi-view representation schemes) or 3D convolutions. These networks are able to achieve similar or better classification accuracy when compared to volumetric approaches, and their efficiency is unmatched.

PointNets have been successful in a myriad of different classification and segmentation tasks. Perhaps most interesting for this work is their use in the recognition of objects in scenes taken from self-driving cars \cite{FrustrumPointNets}. Previous work in point cloud recognition was completed without deep learning in \cite{Wang2015RSS}. 

\  \\

In this work we do not consider multi-view representations. Firstly, multi-view representations convert 3D data to a collection of 2D images, thus making them compatible with existing methods for robustness analysis of image classifiers (e.g. \cite{JSMA, CW-attacks, propertiesOfNeuralNetworks, MadryPGD}). Further, multi-view networks require full 3D information about an object under consideration which is rarely available when operating in a real-time scenario.
\footnote{Despite using the term multi-view, \cite{multiview-driving} is really referring to a fusion of multi-modal views, not multiple unimodal views.} The difficulty with the simultaneous analysis of these approaches is their vastly different architectural composition. In order to rectify this, we will unify both approaches under the following framework (which holds true for volumetric and order-invariant network architectures): 

\centerline{Data $\mapsto$ Latent Translation $\mapsto$ Pooling $\mapsto$ FCN}

\noindent
where FCN stands for fully connected network and refers to a neural network with potentially several layers of neurons which are fully connected. The specifics of this unifying framework for 3D deep learning will be presented in detail in Section \ref{sec:theory} and examples of different representations \cite{VoxNet2015,PointNetOrig} are given in Figure \ref{fig:background}.

\subsection{Safety of Deep Learning}\label{sec:safetyofdeeplearning}

The phenomenon of adversarial examples has provoked a growing concern about the safety of deep learning algorithms. In general we can split methods for crafting adversarial examples into classes based on the threat model (i.e. setting of the adversary, see \cite{BlackBoxAttack} for a thorough treatment) and properties of the examples found.

Attacks are split into \textit{white-box} algorithms and \textit{black-box} algorithms depending on what facets of the model an adversary has access to. We say that an algorithm with access to the inputs, outputs, weights and architecture of a model is a \textit{white-box} method as it can look inside of the model to determine a best attack. A \textit{black-box} attack, on the other hand, is only able to query the model under scrutiny, or in some extreme cases the algorithm may only have access to input-output pairs.

Algorithms can be further decomposed based on what kind of guarantees they are able to provide about the adversarial example they craft; if an algorithm is able to guarantee that it finds a minimal adversarial example or can guarantee that an adversarial example does not exist if it cannot find one, then we consider it a \textit{verification} algorithm, as opposed to \textit{heuristic} search algorithms which make no guarantees about the quality or existence of any adversarial examples that are crafted. 

In our review of these methods, we seek to only give a brief summary of pertinent works rather than providing an exhaustive treatment of the field. 

\textbf{White-box Heuristic Algorithms}
One of the first explorations of adversarial examples was reported 
in \cite{propertiesOfNeuralNetworks}, which framed the discovery of adversarial examples as a constrained optimization problem in the $l_2$ norm. This was followed by \cite{FGSM}  which improved upon the L-BFGS algorithm proposed in \cite{propertiesOfNeuralNetworks} and expanded the attack to the $l_\infty$ metric. The current state of the art in white-box attack methods, however, is CW-attacks \cite{CW-attacks}, which uses a different optimization problem to generate more refined (i.e. more similar to the original input) adversarial examples.

\textbf{White-box Verification Algorithms}
An early attempt at verifying neural networks uses a simplification of the classifier as a linear system and formulates the verification procedure as the potential solution to a set of linear constraints. More recent work has expanded this approach successfully to rectified linear units by employing an extension of the simplex method to solve the system of equations \cite{reluplex}. Other methods use different iterative refinements in order to find adversarial examples or prove that none exist. An early attempt in this vein (DLV) uses a multi-path search through the connections of the network to exhaustively explore a region around the input through finite discretisation \cite{safetyVerification}. Another white-box verification approach employs global optimisation 
\cite{Ruan2018Reachability}. 

\textbf{Black-box Algorithms}
One of the first black-box methods for discovering adversarial examples involved training a surrogate model and then applying a white-box attack on the surrogate model \cite{BlackBoxAttack}. This approach relies on the transfer of the attack from the surrogate to the real model, but this was shown to be empirically effective. Further, iterative approaches to verification of neural networks have also been done in the black-box setting, such as \cite{FGBBST}, which uses exhaustive input layer explorations to formulate optimal $l_0$ attacks on images. This method was refined and improved in \cite{GameBasedVerification} by exploiting Lipschitz continuity. 

To the best of our knowledge, the work on robustness of 3D deep learning pipelines has focused entirely on randomized occlusion. In \cite{PointNetOrig}, they make a specific claim of robustness that stems from the existence of a critical set. In this work we will generalize the idea of a critical set to volumetric networks and will reverse their claim to show that, while in theory critical sets may offer robustness, in practice they are actually a weak point that can be exploited. Further discussion 
appears in Section \ref{sec:discussion}.

\section{Robustness Analysis} \label{sec:theory}

In this section we formalize the ideas and notation that will allow us to analyze the robustness of 3D deep learning pipelines.

\subsection{Representations for 3D Data}\label{sec:prelims}

We take a neural network $N: \mathbf{X} \mapsto \mathbf{Y}$ to be a function with domain $\mathbf{X} $ and co-domain $\mathbf{Y}$. An input or object is a set of vectors $x = \{x_0, ..., x_n \}$ where $x_i \in \mathbb{R}^3$. For the remainder of this paper, we will assume the domain to be a set of such sets, $x \in \mathbf{X} \subseteq 2^{\mathbb{R}^3_{[0,1]}}$. 
Moreover, we will define $\mathbf{Y}$ to be the set of possible classes for each object. The output of the network with respect to $x$ is given as $N(x) = y$, for some $y \in \mathbf{Y}$. Further, we will represent the network assigned probability (or confidence) that $x$ belongs to a class $y$ as $N_y(x)$. Finally, we use $|\cdot|$ to denote the cardinality (i.e. number of unique elements) of a set (where elements belong to $\mathbb{R}^3$). 

In Section \ref{sec:3DDL}, we referenced the fact that almost all deep 3D classification algorithms can be broken down into two functions, one which translates the input to a latent representation, and another which classifies based on that latent representation.\footnote{Note that this is explicitly done in \cite{PointNetOrig} in order to gain an order-invariant input.} We will take the latent translation of an input $x$ to be represented by $L(x) = l$ where $l$ is the result of a max pooling operation on the $d$-dimensional latent vector (specifics are given below). After this latent translation, we use $y = M(l)$ to represent the output of the max pooled latent vector $l$ from an FCN $M$. In summary, this means we break down the original neural network, $N$, into a composition of functions, $M(L(x))$.

\textbf{PointNets} As discussed in Section \ref{sec:3DDL}, PointNets are designed to work on raw point cloud data. Formally, the input to a PointNet, $x$, is (as in our preliminaries) a set of $n$ points from $\mathbb{R}^3$ normalized to the unit cube. One major challenge with dealing with this kind of data is the fact that there are $n!$ possible orderings of a single input. As such, PointNets must be symmetric functions (i.e. order invariant). This is achieved using the latent translation of the data. In essence, each component of the input, $x_i$ (a point in $\mathbb{R}^3$), is passed through a series of translations to higher dimensions followed by a convolution operation. In the original PointNet paper, \cite{PointNetOrig}, each point $x_i \in \mathbb{R}^3$ was translated into $\mathbb{R}^{64}$ via a CNN (convolutional neural network) and then this 64-dimensional vector was subjected to a 1D convolution operation. A similar procedure was repeated until the network arrived at $n$-many 1024-dimensional latent vectors, at which point all $n$ 1024-dimensional vectors are max pooled into a single representative 1024-dimensional latent vector, $l$. The process we just described will (as previously mentioned) be referred to as $L_{point}: \mathbb{R}^{3\times n}_{[0,1]} \mapsto \mathbb{R}^{1024}$. The original PointNet architecture is presented in Figure \ref{fig:background}.

\textbf{Volumetric Networks} When compared to PointNets, volumetric networks have a much more canonical latent vector translation. The first step in a volumetric network, given that the input is a set of vectors from $\mathbb{R}^3$, is the translation into a voxelized cube (described in Section \ref{sec:3DDL}). In the case of VoxNet \cite{VoxNet2015}, the input cube $x \in \mathbb{R}^{d \times d \times d}_{[0,1]}$ is passed through a three-dimensional convolution operator to produce $m$ different filters for the cube; this is repeated a number of times and then the resulting series of $m$-many $s \times s \times s$ cubes is passed into a fully connected network. As such, we take the flattened version of the $m$-many scaled down cubes to be the output of the latent translation $L_{volum}$. Of course, straightforward convolutional neural networks are not the only kind of volumetric network that have been studied. Many popular forms of 2D CNNs have been scaled directly up to 3D. An example of the decomposition of a volumetric 3D pipeline can also be seen in Figure \ref{fig:background}.

\  \\

\begin{figure}[t]
\begin{center}
 \includegraphics[width=0.85\linewidth]{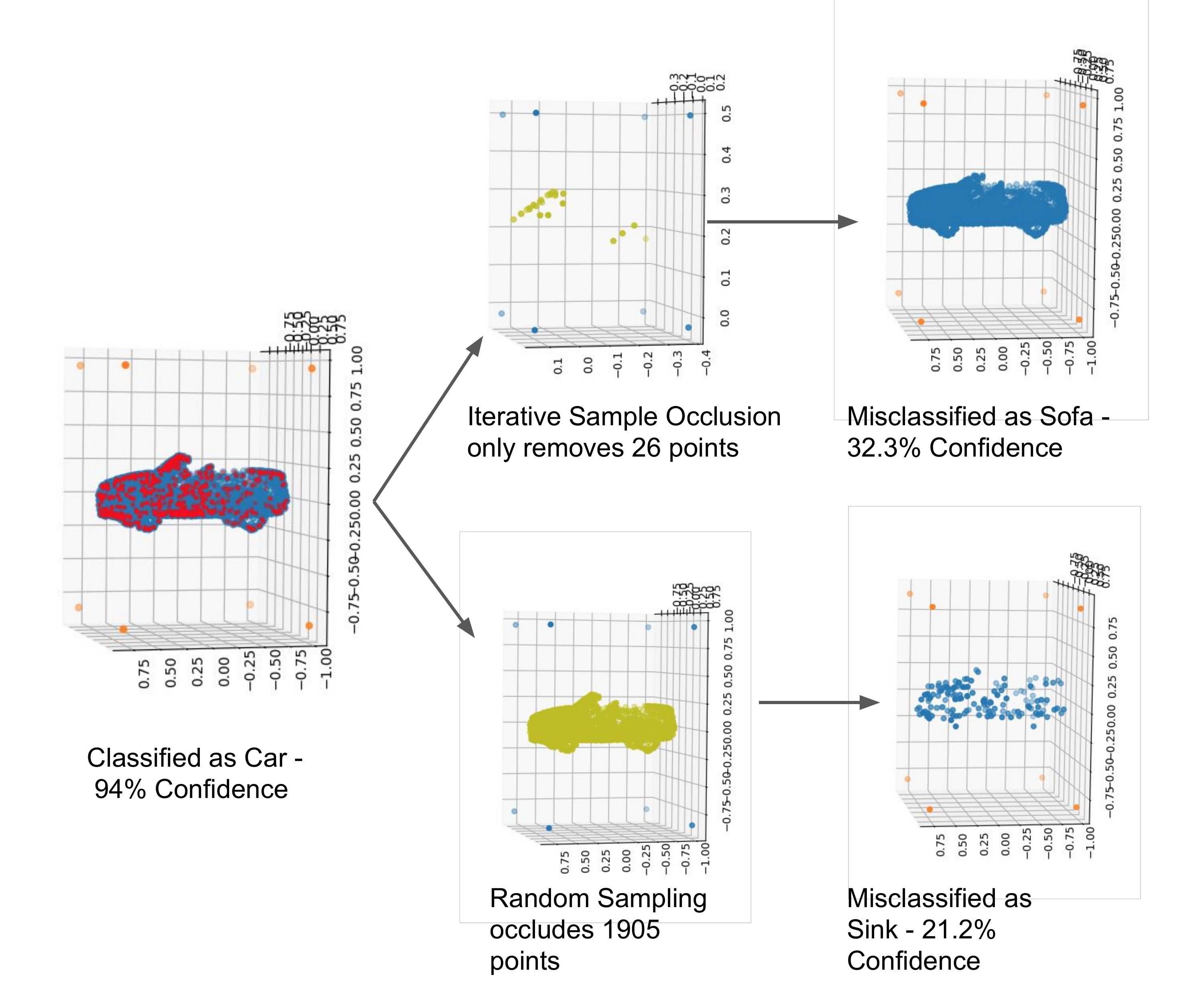}
\end{center}
\vspace*{-0.5cm}
   \caption{A car, though initially classified correctly with high confidence, is easily changed into being classified as a sofa with only 26 points changed. The original input has been marked with red points to denote all the parts of the input that exist in the critical set.}
\label{fig:pointnetcar}
\end{figure}

Given our decomposition of 3D deep learning pipelines into latent translations and fully connected networks (FCNs), it seems straightforward to simply attack the FCN with standard robustness analysis techniques and then project the latent vector back into the ambient space. Unfortunately, almost all current methods of crafting adversarial examples are inappropriate for safety testing of 3D deep learning pipelines. The most popular methods of generating adversarial examples use either the $l_\infty$ or $l_2$ norm. Optimization with respect to these kinds of manipulations encourages a change in all (or almost all) of the components of an input by some small value $\pm \epsilon$. This is inappropriate for the volumetric representation because each component of the input represents the probability that that component is inhabited, and assigning every non-inhabited point in a scene $\pm \epsilon$ probability of having an occupying object is unrealistic. Further, in the case of a point-cloud representation scheme, using a blanket manipulation of every point in the input (in a different direction) has the potential to take the input outside of the natural data manifold, that is to say, that if $\epsilon$ is non-negligible, then we may corrupt the underlying structure of the input that makes it recognizable to humans. Examples of this phenomenon exist in natural images for moderate values of $\epsilon$ \cite{FoolingHumans}, and because point clouds are much more difficult for humans to recognize without manipulation, it seems that a blanket change to all points would, in practice, render many point clouds unrecognizable to humans. 

Taking this as the case, we turn to $l_0$-norm optimized attack algorithms. An $l_0$-norm attack simply optimizes for the number of changes made to the input. We note that the cardinality operator $|\cdot|$ allows a subset of the manipulations allowed by the $l_0$ norm. This prioritizes maintaining the underlying structure of the point cloud while allowing for occlusions, introduction of spurious features, and a handful of shifts in data positions. In the next section we will detail the formalization of the problem of crafting adversarial examples with adversarial occlusion on 3D data.

\begin{figure}[t]
\begin{center}
 \includegraphics[width=0.88\linewidth]{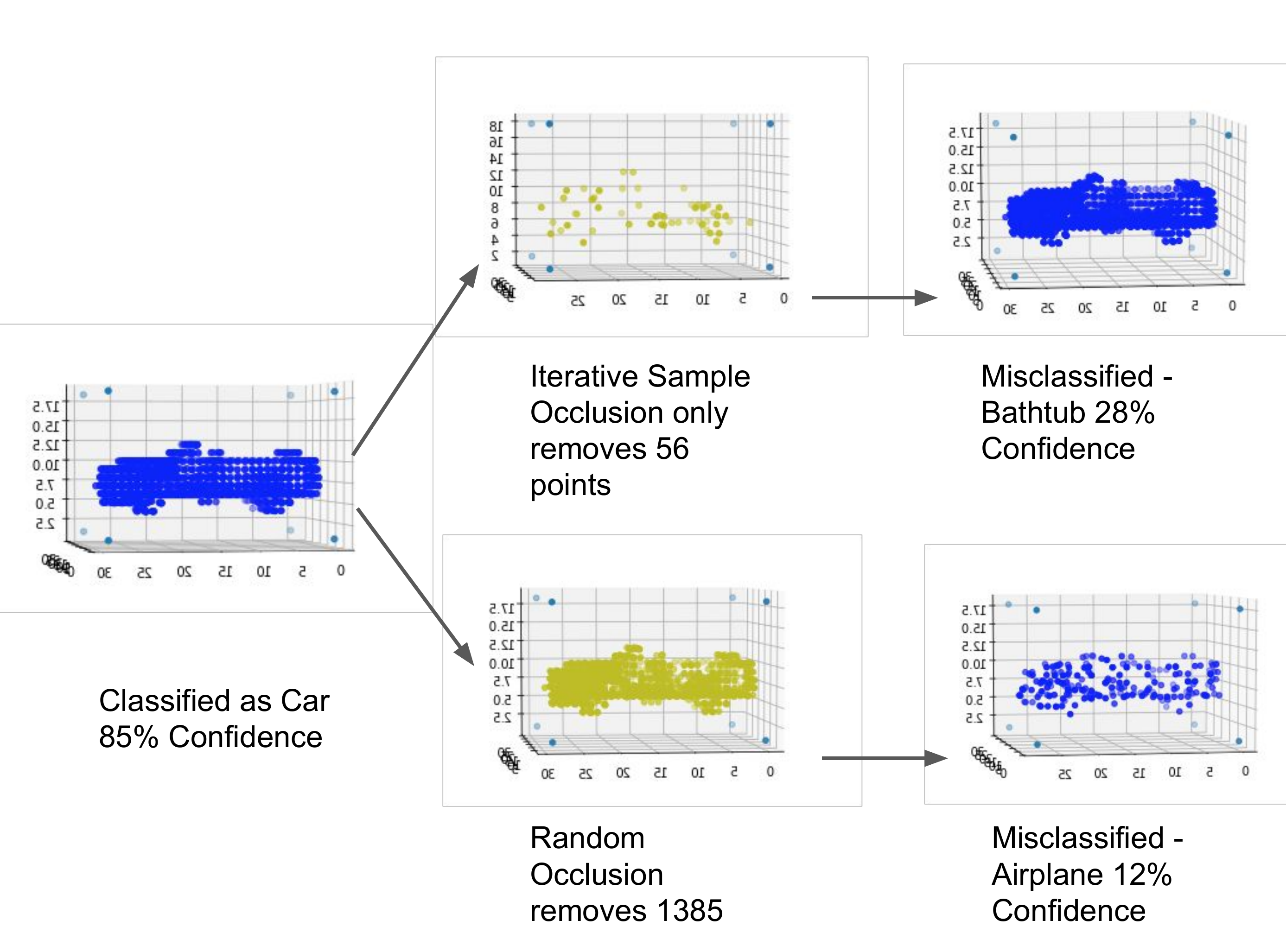}
\end{center}
\vspace*{-0.5cm}
   \caption{We have taken the same model as in Figure \ref{fig:pointnetcar} and generated an adversarial example of the input with respect to the VoxNet architecture. The details of the model salience can be found in Figure \ref{fig:voxsal}.}
\label{fig:voxnetcar}
\end{figure}

\subsection{Occlusion Attacks}
In tasks such as 2D pedestrian recognition and facial recognition, occlusion has been studied under several different models \cite{OcclusionModel1, OcclusionModel2}. Occluded inputs are valuable to consider when evaluating the safety of a decision making network as it is not always the case that pedestrians, vehicles, or objects of interest will be in complete and plain view. These occlusions may not always be natural, however. In fact, it is often the case that data (point clouds) from sensors exhibit stochasticity. As such, it is imperative to model the cases in which some parts of the data might be missing. 

Because the point-cloud and volumetric domains are less rich than the image domain, there is no way to directly utilize pre-developed models of occlusion. In lieu of this, we simply define an occlusion of $x$ to be some $x' \subset x$ and set up an optimization to find the minimum occlusion defined as:

\begin{equation}\label{eq:OcclusionObjective}
    \argmin_{x' \subseteq x} (|x| - |x'|) \quad s.t. \quad  N(x') \neq N(x)
\end{equation}

Of course, where we have an $n$ point input there are $2^n-1$ possible occlusions that could possibly satisfy this objective. In order to cut down this search space we will use the notion of a latent translation to define a critical point set, $C^x_n$ (abbrev. $C_n$), which is a concept introduced in \cite{PointNetOrig}. We again distinguish our use of the critical set by noting that in this work we generalize it to volumetric networks and use it explicitly to display model weakness rather than to hypothesize robustness:

\begin{equation}\label{eq:CriticalPointSet}
    C_n^x = \{ x_i \in x \; \mid \; L(x/x_i) \neq L(x) \}
\end{equation}
\noindent
where we define $x/x_i$ to be the values of $x$ that do not include $x_i$. That is, a value exists in the critical point set if and only if its removal from the input impacts the latent representation of the input. For $L_{point}$ we know that $x_i \in C^x_n$ if $\forall i\neq j, l_i \geq l_j$ by virtue of the max pooling layer. Further, we know that for $L_{volum}$, $x_i \in C^x_n$ if $\forall l_j \in \eta(l_i), \alpha(l_i) \geq \alpha(l_j)$ where $\alpha(l_i)$ represents the activation value of the $x_i^{th}$ input voxel after convolution\footnote{Note that multiple $x_i$'s will map to the same $\alpha(x_i)$ as a result of the down-sampling associated with the convolution operator.} and $\eta(x_i)$ represents the pooling neighborhood of the last layer of the latent translation prior to the flattening and fully-connected network. 

Both of these formulations require us to use the vector $l$ which implies that the method needs white-box access to the network in order to find adversarial examples; however, we will later prove that, by exploiting our knowledge of 3D deep learning approaches, our attack algorithm can indeed work in the black-box setting.

\textbf{Iterative Salience Occlusion} Given the above ability to calculate the critical point set, we propose the following simple algorithm to randomly explore and iteratively refine an occlusion attack.

\begin{algorithm}
\caption{Iterative Salience Occlusion}\label{euclid}
\begin{algorithmic}[1]
\Procedure{ISO}{N, y, x, $g$}
\State $x' \gets x$
\While{$g(N, y, x, x')$}
\State $C_n \gets $Calc$C_n(N, x')$
\State $C_n \gets $Rank$(C_n)$
\For {$x_i \in C_n$}
\State \textbf{if} ($N(x') \neq y$)
\State \quad break
\State \textbf{if}($N_y(x'/x_i) \leq N_y(x')$)
\State \quad $x' \gets x'/x_i$
\EndFor
\For{$x_i \in x - x'$}
\State \textbf{if}($N(x' \cup x_i) \neq y$)
\State \quad $x' \gets x' \cup x_i$
\EndFor
\State \textbf{if}($N(x') \neq y$ and $g(N, y, x, x')$ $\neq$ true)
\State \quad $x' \gets x$ 
\EndWhile
\State return $x'$
\EndProcedure
\end{algorithmic}
\end{algorithm}

Within the ISO 
algorithm we take the Rank function to be some way of computing and ordering the critical point set based on the saliency and assume that Rank never returns the same permutation twice until it has exhausted all other options. All operations within the algorithm use standard set-theoretic notation.

Our 
algorithm has several properties that make it ideal as a method for evaluating worst-case occlusion. Firstly, the algorithm is \textit{anytime}, meaning that the algorithm can terminate given any user defined termination condition, which we encode as a boolean function $g$, an input to the algorithm. The function $g$ can encode adversarial goals including confidence reduction ($N_y(x) - N_y(x') > k$), crafting adversarial examples ($N(x) \neq N(x')$), or crafting targeted adversarial examples ($N(x') = y'$). One may also change the algorithm slightly so that it returns the best adversarial example it has found in a specified amount of time. 

In addition to being \textit{anytime}, we can show that, in the case of PointNet architectures, this algorithm (which is currently \textit{white-box}) can operate in the \textit{black-box} setting. 

\begin{theorem}
Given an input $x$ and a PointNet network $N$, we can compute the critical set $C_n$ in a black-box setting, given that all weights in $M$ are non-zero.
\end{theorem}

\begin{proof}
For each $x_i \in x$ we can determine if it exists in the critical set by removing $x_i$ from $x$ and checking the final output of the network. If $x_i \in C_n$, then, by definition, the output will change due to the elimination of its contribution to $l$. If $x_i \not\in C_n$ then the output will not change because $l$ will not change.
\end{proof}


Finally, this algorithm has the strength of being a \textit{verification} approach, meaning that it provides guarantees about both (1) finding an adversarial example if one exists and (2) finding an adversarial example that satisfies Equation \ref{eq:OcclusionObjective}  if one exists. Below we show that, if we set $g$ such that it returns true if and only if all possible $C_s$ permutations have been checked, then Equation \ref{eq:OcclusionObjective} must be satisfied. 

\begin{theorem}\label{thm:guarantee}
Given an input $x \in 2^{\mathbb{R}^3_{[0,1]}}$ and a neural network $N$ that satisfies our framework, we can show that the ISO algorithm will find the optimal adversarial example that satisfies Eq. \ref{eq:OcclusionObjective}.
\end{theorem}

\begin{proof}
First, if there exists $s \subseteq x$ that is an adversarial example, we must find it by exhaustive search. This is because the Rank function will allow us to check all possible manipulation orders of the critical set and we set the $g$ function such that the algorithm will not terminate until this is the case. Next, we know that for each manipulation order that is checked we must yield the smallest possible subset for that order because of the iterative refinement on lines 11-13 of the algorithm (any points unnecessarily removed will be added back in). This means that if an adversarial example exists we will find it, and any example we find must be minimal.
\end{proof}

\section{Evaluation} \label{sec:emprical}

Given the strengths of the ISO algorithm, we will use it to show that, in almost all cases, random occlusion vastly overestimates the robustness of 3D deep learning pipelines. This discovery exacerbates the need for further study in the development of robust 3D deep learning algorithms and methods to check their safety.

In order to study the effectiveness of the proposed algorithm on both point-cloud and volumetric representations, we retrained both the VoxNet \cite{VoxNet2015} and PointNet \cite{PointNetOrig} network architectures on the ModelNet10 and ModelNet40 benchmarks \cite{ModelNet} as well as on 3D objects extracted from the LIDAR sensor of the KITTI self-driving car \cite{Geiger2012KITTI} (see Appendix). In all cases, the networks were trained for 50 epochs according to the training details provided in the respective papers. Using the pre-defined test-train split from ModelNet, all trained networks achieved accuracy within 4\% of the reported accuracy on the test set. 

After training each network, we sampled 200 objects from the test set in order to evaluate the robustness of each model. In the case of ModelNet10, the networks were tested with a time cutoff of 2 seconds to find an adversarial example, and in the case of ModelNet40 the ISO algorithm was given a 5 second cutoff. On the other hand, the random occlusion algorithm was simply given a random permutation of the data and removed the data in that random order until a misclassification was found, at which point it would report the number of points it needed to remove. As we can see in Figures
\ref{fig:pointnet-robustnesscurve}.a and \ref{fig:voxnet-robustnesscurve}.a, the results reported by the random occlusion match up very well with those that are reported in \cite{PointNetOrig}. Over the set of 200 objects we track how the accuracy of the network changes as we occlude more points from each model. In the worst case, VoxNet trained on ModelNet40, it took only the occlusion of 6.5\% of the input in order to reduce the network to 0\% classification accuracy. 

\begin{table*}[t]\small
\centering
\begin{tabular}{ |c|c ||c|| c|c|c|c|c||} 
\hline
Architecture & Dataset & Method & 0\% Occl.  & 25\% Occl. & 50\% Occl. & 75\% Occl. & 95\% Occl. \\
\hline
\hline
\multirow{4}{4em}{VoxNet} 
                            & \multirow{2}{6em}{ModelNet10} 
                                                            & Rand. & 79.8\% & 72.1\% & 66.9\% & 51.9\% & 10.9\% \\ 
                                                            \cline{3-8}
                                                            & & ISO & 79.8\% & 1.0\% & 0\% & 0\% & 0\% \\ 
                                                            \cline{2-8}\cline{2-8}
                            & \multirow{2}{6em}{ModelNet40}
                                                            & Rand. & 76.1\% & 60.9\% & 39.1\% & 12.3\% & 2.0\%  \\
                                                            \cline{3-8}
                                                            & & ISO & 76.1\% & 0\% & 0\% & 0\% & 0\% \\ 
                                                             \cline{2-8}\cline{2-8}
                            & \multirow{2}{6em}{KITTI}
                                                            & Rand. & 71.5\% & 55.5\% & 30.0\% & 13.5\% & 6.5\%  \\
                                                            \cline{3-8}
                                                            & & ISO & 71.5\% & 0\% & 0\% & 0\% & 0\% \\ 
                                                            \cline{3-8}
\hline
\multirow{4}{4em}{PointNet} 
                            & \multirow{2}{6em}{ModelNet10}
                                                            & Rand. & 86.1\% & 84.0\% & 83.0\% & 79.2\% & 60.2\%  \\
                                                            \cline{3-8}
                                                            & & ISO & 86.1\% & 27\% & 5\% & 0\% & 0\% \\ 
                                                            \cline{2-8}\cline{2-8}
                            & \multirow{2}{6em}{ModelNet40} 
                                                            & Rand. & 82.5\% & 79.4\% & 78.2\% & 74.3\% & 35.9\% \\
                                                            \cline{3-8}
                                                            & & ISO & 82.5\% & 17.9\% & 4.1\% & 0\% & 0\% \\ 
                                                            \cline{2-8}\cline{2-8}
                            & \multirow{2}{6em}{KITTI} 
                                                            & Rand. & 73.0\% & 72.5\% & 72.0\% & 68.5\% & 40.5\% \\
                                                            \cline{3-8}
                                                            & & ISO & 73.0\% & 49.5\% & 37.5\% & 0\% & 0\% \\ 
\hline
\end{tabular}
\caption{Reduction in classification accuracy for different levels of both random and iterative saliency occlusion for all tested datasets (ModelNet10, ModelNet40 and KITTI).}
\label{table:comparison}
\end{table*}

A more readily interpretable version of Figures
\ref{fig:pointnet-robustnesscurve} and \ref{fig:voxnet-robustnesscurve} exists in Table \ref{table:comparison}. We see that, in general, using random occlusion instead of the ISO algorithm overestimates robustness of the network to occlusion by around 60\%. Further, we see that each of the models was reduced to less that 10\% classification accuracy providing that the ISO algorithm was given time to manipulate half of the data.

\section{Discussion}\label{sec:discussion}

From our theoretical analysis it seems that one of the key components in accurately analyzing a 3D deep learning architecture is the cardinality of the critical point set. On one hand, it may seem desirable to have a network with a low critical set cardinality, that is to say, a network that only looks at several key points of the input in order to make a decision. It is clear that there exists a relationship between the number of points in the critical set and the effectiveness of a random occlusion approach. For example, the probability that a random network will select a point that matters to the classification of any given input point cloud is precisely equal to the cardinality of the critical set divided by the cardinality of the unique points in the input. On the other hand, it seems that a low critical point set cardinality will lead directly to the success of the proposed ISO algorithm as it only manipulates points that exist in the critical set.

\textbf{PointNet}
  In Figure \ref{fig:criticalpoint-cards} we see that the the cardinality of the critical set on ModelNet inputs with 2048 points hovers around 450 to 500 points, roughly a quarter of the data. 

\begin{figure}[t]
\begin{center}
 \includegraphics[width=0.85\linewidth]{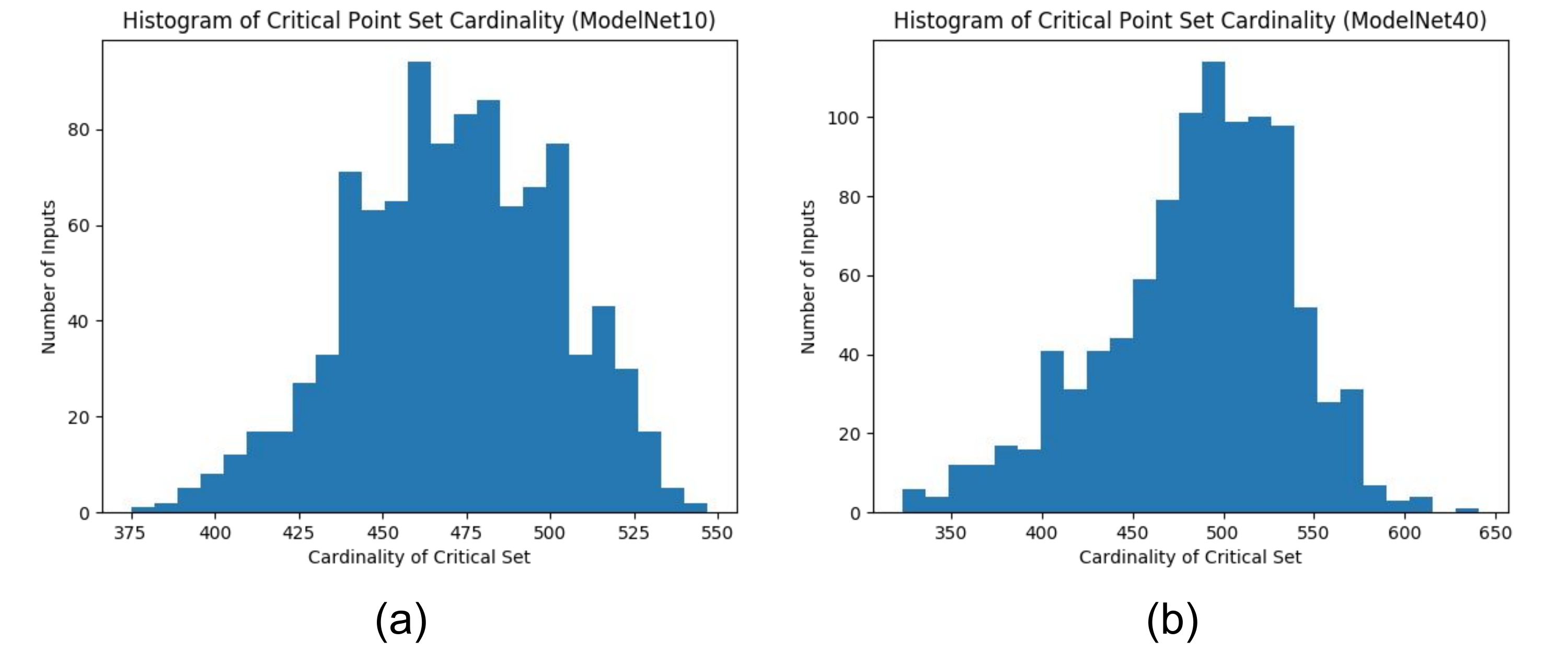}
\end{center}
\vspace*{-0.5cm}
   \caption{Distributions of critical set cardinalites for PointNet trained on ModelNet10 and ModelNet40. These distributions were obtained by the calculation of critical point set cardinalites on 1000 test set models. }
\label{fig:criticalpoint-cards}
\end{figure}

Given that the average critical point set cardinality is about 25\% of the input model, we would like to highlight the performance of the ISO algorithm after removal of 25\% of the data. The reason that the network does not immediately collapse to 0\% classification accuracy is that, after a point is removed from the critical point set, there is an opening for a new point to become critical. Despite this complication, our simple algorithm allows us to get an exponential decrease in accuracy with an increase in occlusion, as seen in plots in Figures
\ref{fig:pointnet-robustnesscurve} and \ref{fig:voxnet-robustnesscurve}. Ultimately, this increases our confidence in the link between the modifications of the critical point set and the network's overall robustness.

\begin{figure}[t]
\begin{center}
 \includegraphics[width=1.0\linewidth]{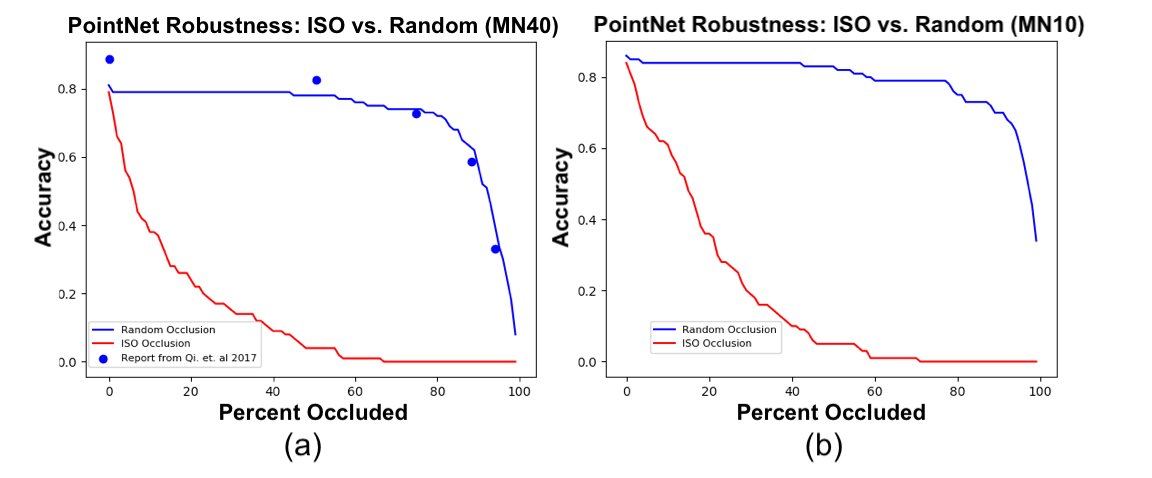}
\end{center}
\vspace*{-0.5cm}
   \caption{PointNet robustness on ModelNet10 (b) and ModelNet40 (a). Blue plots the change in accuracy due to random occlusion whilst red is the change in accuracy due to ISO occlusion. In (a), we have also pulled values from the identical random occlusion test performed in \cite{PointNetOrig}.}
\label{fig:pointnet-robustnesscurve}
\end{figure}

\textbf{VoxNet} While defining the critical set for PointNet is straightforward, VoxNet (and volumetric approaches in general) pose a more difficult problem. Unlike PointNet, where membership in the critical set is binary, convolutional volumetric networks yield continuous salience values. This means that every point in the input exists in the critical set. To rectify this, we set a threshold that determines which points are critical or not. In order to be consistent with the analysis for PointNet, we have set the threshold for membership of the critical set to be the most salient 25\% of the input. This is visualized in Figure \ref{fig:voxsal}; however, when we look at the 25\% occluded point in Table \ref{table:comparison}, we see that we can force almost all examples to be misclassified.

\begin{figure}[t]
\begin{center}
 \includegraphics[width=1.0\linewidth]{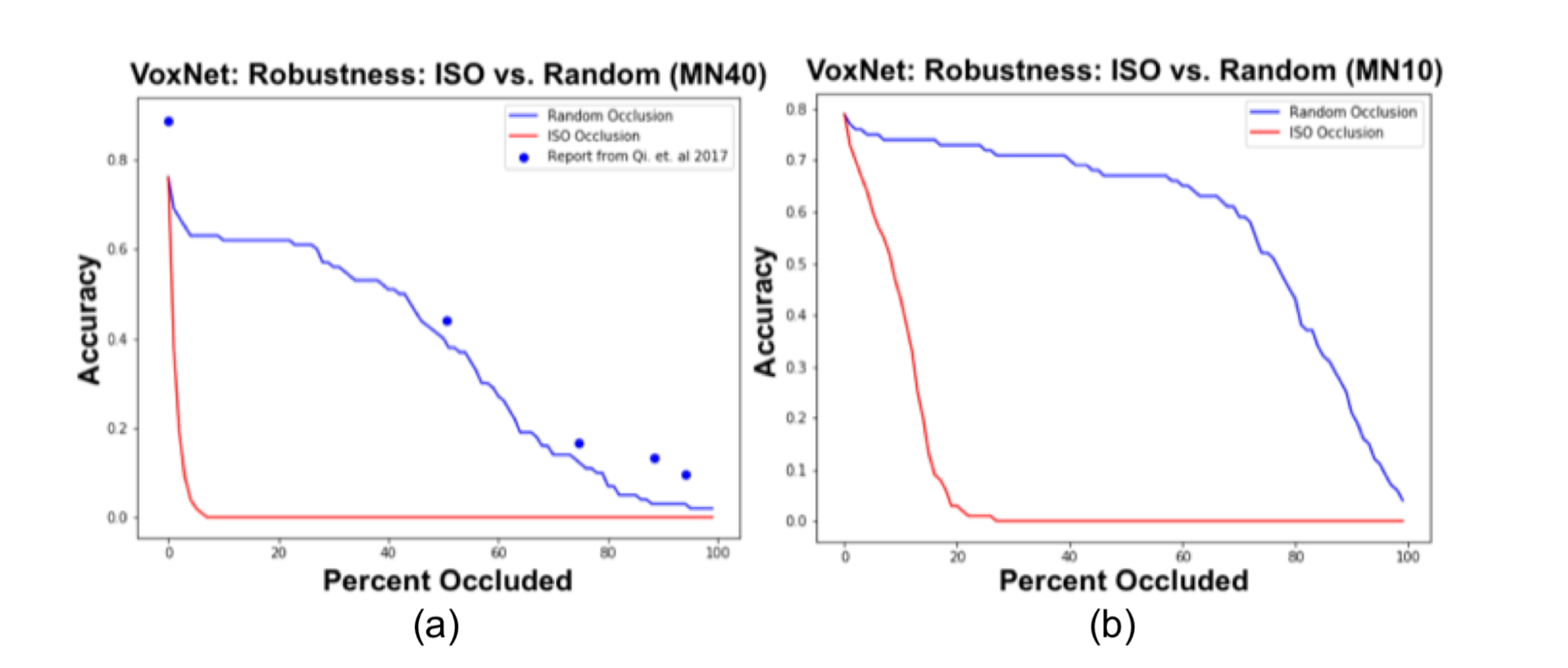}
\end{center}
\vspace*{-0.5cm}
   \caption{Figure follows the same format as Figure \ref{fig:pointnet-robustnesscurve}. VoxNet robustness on ModelNet10 (b) and ModelNet40 (a). On the ModelNet40 benchmark, it takes roughly 6.5\% occlusion to reduce the network performance to 0\%.}
\label{fig:voxnet-robustnesscurve}
\end{figure}

\begin{figure}[t]
\begin{center}
 \includegraphics[width=0.80\linewidth]{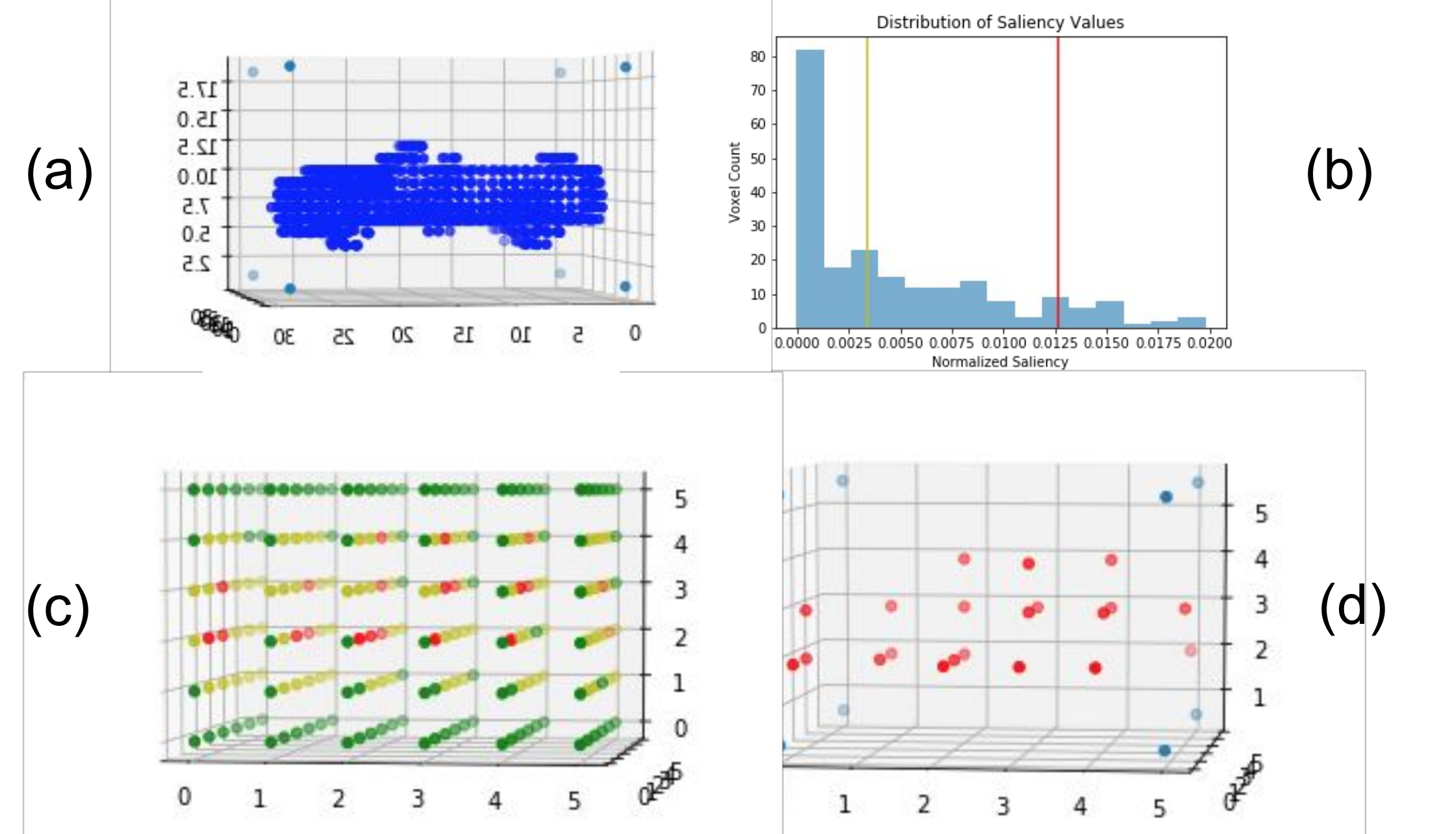}
\end{center}
\vspace*{-0.5cm}
   \caption{For voxelized inputs, (a), we calculate the saliency as described in Section \ref{sec:prelims}. We normalize the latent saliency of the points into a salience distribution. The different quartiles (q1 - yellow line, q3 - red line in (b)) give us an idea of relative saliency; in (c) we map each point in the input to its saliency where green represents the least salient and red the most salient. The critical set is then considered to be any points that fall in the upper 25\% of all computed saliency values, this has been visualized in (d). }
\label{fig:voxsal}
\end{figure}

\textbf{Future Directions} One aim of the approach presented in this paper is to introduce and encourage the use of alternate metrics for accuracy when evaluating 3D deep learning pipelines that may come into use in real-time, safety-critical scenarios. The methods of computing saliency that are formulated in this work can be directly utilized in the frameworks of \cite{FGBBST, GameBasedVerification, Cardelli2019Bounds} in order to derive bounds on the safety of classification. The improvement of generalization of 3D deep learning would 
be aided greatly 
by formulating more robust pipelines for processing point-cloud data.
\vspace{-0.25cm}
\section{Conclusion}

In this work, we demonstrate that the critical point sets induced by the latent space translation in 3D deep learning pipelines, for both point-cloud and volumetric representations, expose a vulnerability to adversarial occlusion attacks that to this point had not been studied. We show that, in the worst case, a black-box verification approach can, using only 4 seconds per input example, reduce the accuracy of a network to 0\% despite manipulating at most 6.5\% of the input. 

\section*{Acknowledgements}\small
This work has been partially supported by 
the EPSRC Programme Grant on Mobile Autonomy
(EP/M019918/1).

\newpage
{\small
\bibliographystyle{ieee}
\bibliography{egbib}

\begin{thebibliography}{10}\itemsep=-1pt

\bibitem{Cardelli2019Bounds}
L.~{Cardelli}, M.~{Kwiatkowska}, L.~{Laurenti}, N.~{Paoletti}, A.~{Patane}, and
  M.~{Wicker}.
\newblock {Statistical Guarantees for the Robustness of Bayesian Neural
  Networks}.
\newblock {\em arXiv e-prints}, Mar. 2019.

\bibitem{CW-attacks}
N.~{Carlini} and D.~{Wagner}.
\newblock Towards evaluating the robustness of neural networks.
\newblock In {\em 2017 IEEE Symposium on Security and Privacy (SP)}, pages
  39--57, May 2017.

\bibitem{multiview-driving}
X.~Chen, H.~Ma, J.~Wan, B.~Li, and T.~Xia.
\newblock Multi-view 3d object detection network for autonomous driving.
\newblock pages 6526--6534, 07 2017.

\bibitem{ALargeDatasetOfObjectScans}
S.~Choi, Q.~Zhou, S.~Miller, and V.~Koltun.
\newblock A large dataset of object scans.
\newblock {\em CoRR}, abs/1602.02481, 2016.

\bibitem{SydneyDataset}
M.~De~Deuge, A.~Quadros, C.~Hung, and B.~Douillard.
\newblock Unsupervised feature learning for classification of outdoor 3d scans.
\newblock {\em Australasian Conference on Robotics and Automation, ACRA}, 01
  2013.

\bibitem{FoolingHumans}
G.~F. Elsayed, S.~Shankar, B.~Cheung, N.~Papernot, A.~Kurakin, I.~J.
  Goodfellow, and J.~Sohl{-}Dickstein.
\newblock Adversarial examples that fool both human and computer vision.
\newblock {\em CoRR}, abs/1802.08195, 2018.

\bibitem{Geiger2012KITTI}
A.~Geiger, P.~Lenz, and R.~Urtasun.
\newblock Are we ready for autonomous driving? the kitti vision benchmark
  suite.
\newblock In {\em CVPR 2012}, 2012.

\bibitem{FGSM}
I.~Goodfellow, J.~Shlens, and C.~Szegedy.
\newblock Explaining and harnessing adversarial examples.
\newblock In {\em International Conference on Learning Representations}, 2015.

\bibitem{NYT-Uber}
T.~Griggs and D.~Wakabayashi.
\newblock How a self-driving uber killed a pedestrian in arizona, Mar 2018.

\bibitem{FusionNet}
V.~Hegde and R.~Zadeh.
\newblock Fusionnet: 3d object classification using multiple data
  representations.
\newblock {\em CoRR}, abs/1607.05695, 2016.

\bibitem{safetyVerification}
X.~Huang, M.~Kwiatkowska, S.~Wang, and M.~Wu.
\newblock Safety verification of deep neural networks.
\newblock In {\em CAV}, pages 3--29, Cham, 2017. Springer International
  Publishing.

\bibitem{reluplex}
G.~Katz, C.~W. Barrett, D.~L. Dill, K.~Julian, and M.~J. Kochenderfer.
\newblock Reluplex: An efficient {SMT} solver for verifying deep neural
  networks.
\newblock In {\em CAV}. Springer International Publishing, 2017.

\bibitem{MadryPGD}
A.~Madry, A.~Makelov, L.~Schmidt, D.~Tsipras, and A.~Vladu.
\newblock Towards deep learning models resistant to adversarial attacks.
\newblock In {\em International Conference on Learning Representations}, 2018.

\bibitem{OcclusionModel1}
M.~Mathias, R.~Benenson, R.~Timofte, and L.~V. Gool.
\newblock Handling occlusions with franken-classifiers.
\newblock In {\em 2013 IEEE International Conference on Computer Vision}, pages
  1505--1512, Dec 2013.

\bibitem{dronelanding}
D.~Maturana and S.~Scherer.
\newblock 3d convolutional neural networks for landing zone detection from
  lidar.
\newblock In {\em 2015 IEEE International Conference on Robotics and Automation
  (ICRA)}, pages 3471--3478, May 2015.

\bibitem{VoxNet2015}
D.~Maturana and S.~Scherer.
\newblock {VoxNet: A 3D Convolutional Neural Network for Real-Time Object
  Recognition}.
\newblock In {\em {IROS}}, 2015.

\bibitem{OcclusionModel2}
J.~Noh, S.~Lee, B.~Kim, and G.~Kim.
\newblock Improving occlusion and hard negative handling for single-stage
  pedestrian detectors.
\newblock In {\em CVPR 2018}, June 2018.

\bibitem{BlackBoxAttack}
N.~Papernot, P.~McDaniel, I.~Goodfellow, S.~Jha, Z.~B. Celik, and A.~Swami.
\newblock Practical black-box attacks against machine learning.
\newblock In {\em Proceedings of the 2017 ACM on Asia Conference on Computer
  and Communications Security}, ASIA CCS '17, pages 506--519. ACM, 2017.

\bibitem{JSMA}
N.~Papernot, P.~D. McDaniel, S.~Jha, M.~Fredrikson, Z.~B. Celik, and A.~Swami.
\newblock The limitations of deep learning in adversarial settings.
\newblock {\em 2016 IEEE European Symposium on Security and Privacy
  (EuroS\&P)}, pages 372--387, 2016.

\bibitem{FrustrumPointNets}
C.~R. Qi, W.~Liu, C.~Wu, H.~Su, and L.~J. Guibas.
\newblock Frustum pointnets for 3d object detection from {RGB-D} data.
\newblock In {\em {CVPR} 2018.}, pages 918--927, 2018.

\bibitem{PointNetOrig}
C.~R. Qi, H.~Su, K.~Mo, and L.~J. Guibas.
\newblock Pointnet: Deep learning on point sets for 3d classification and
  segmentation.
\newblock In {\em 2017 {IEEE} Conference on Computer Vision and Pattern
  Recognition, {CVPR} 2017}, pages 77--85, 2017.

\bibitem{Qi2017PointNetPlus}
C.~R. Qi, L.~Yi, H.~Su, and L.~J. Guibas.
\newblock Pointnet++: Deep hierarchical feature learning on point sets in a
  metric space.
\newblock In {\em Advances in Neural Information Processing Systems 30: Annual
  Conference on Neural Information Processing Systems 2017.}, pages 5105--5114,
  2017.

\bibitem{Ruan2018Reachability}
W.~Ruan, X.~Huang, and M.~Kwiatkowska.
\newblock Reachability analysis of deep neural networks with provable
  guarantees.
\newblock In {\em {IJCAI} 2018.}, pages 2651--2659, 2018.

\bibitem{MultiViewCNN2015}
H.~{Su}, S.~{Maji}, E.~{Kalogerakis}, and E.~{Learned-Miller}.
\newblock Multi-view convolutional neural networks for 3d shape recognition.
\newblock In {\em 2015 IEEE International Conference on Computer Vision
  (ICCV)}, pages 945--953, Dec 2015.

\bibitem{propertiesOfNeuralNetworks}
C.~Szegedy, W.~Zaremba, I.~Sutskever, J.~Bruna, D.~Erhan, I.~J. Goodfellow, and
  R.~Fergus.
\newblock Intriguing properties of neural networks.
\newblock {\em CoRR}, abs/1312.6199, 2013.

\bibitem{Wang2015RSS}
D.~Z. Wang and I.~Posner.
\newblock Voting for voting in online point cloud object detection.
\newblock In {\em Proceedings of Robotics: Science and Systems}, Rome, Italy,
  July 2015.

\bibitem{Wang2018DGCNN}
Y.~{Wang}, Y.~{Sun}, Z.~{Liu}, S.~E. {Sarma}, M.~M. {Bronstein}, and J.~M.
  {Solomon}.
\newblock {Dynamic Graph CNN for Learning on Point Clouds}.
\newblock {\em arXiv e-prints}, Jan. 2018.

\bibitem{FGBBST}
M.~Wicker, X.~Huang, and M.~Kwiatkowska.
\newblock Feature-guided black-box safety testing of deep neural networks.
\newblock In {\em Tools and Algorithms for the Construction and Analysis of
  Systems - 24th International Conference, {TACAS} 2018}, pages 408--426, 2018.

\bibitem{GameBasedVerification}
M.~{Wu}, M.~{Wicker}, W.~{Ruan}, X.~{Huang}, and M.~{Kwiatkowska}.
\newblock {A Game-Based Approximate Verification of Deep Neural Networks with
  Provable Guarantees}.
\newblock {\em ArXiv e-prints}, July 2018.

\bibitem{ModelNet}
Z.~{Wu}, S.~{Song}, A.~{Khosla}, and J.~{Xiao}.
\newblock 3d shapenets: A deep representation for volumetric shapes.
\newblock In {\em (CVPR) 2015}, pages 1912--1920, June 2015.

\bibitem{VoxelNet}
Y.~Zhou and O.~Tuzel.
\newblock Voxelnet: End-to-end learning for point cloud based 3d object
  detection.
\newblock In {\em {CVPR} 2018}, pages 4490--4499, 2018.

\end{thebibliography}
}

\newpage
\appendix
\section{Further KITTI Evaluation}

The KITTI Dataset \cite{Geiger2012KITTI} is a collection of sensor readings from an autonomous vehicle in an urban environment. This includes point clouds from a Velodyne sensor that have been hand labeled with 3D bounding boxes. To assess classification networks on this data, we isolated all of the objects in each driving scene and created a classification task (determining if an object was a car, truck or other) comprised of 5,000 real-world point clouds. In order to extract these objects we parsed through driving sequences and isolated the point clouds from human-labeled bounding boxes that had sufficient density (i.e. easily recognizable objects).

In Table \ref{tab:table1} we show the initial accuracy of each method on the new data set in the first column. In subseqent columns we show how data occlusion affects the performance of the network (identical to the rows in the main paper). This is done for both white-box and black-box settings. The times reported are the mean times in seconds between the initial classification and the final adversarial example being returned. These times only reflect inputs which were initially classified correctly by the network.

\begin{table}[H]\footnotesize
\begin{center}
\begin{tabular}{|c|c|c|c|c|c|}
\hline
Architecture & 0\% Occl. & 25\% & 50\% & 75\% & time (s)  \\
\hline\hline
PointNet &  73\% & 49.5\% & 37.5\% & 0\% & 5.06 \\
PointNet-BB & 73\% & 49.5\% & 37.5\% & 0\% & 24.28 \\
PointNet-Rand & 73\% & 72.5\% & 72.0\% & 68.5\% & 3.54 \\
VoxNet & 71.5\% & 0\% & 0\% & 0\% & 1.16 \\
VoxNet-BB & 71.5\% & 0\% & 0\% & 0\% & 1.44 \\
VoxNet-Rand & 71.5\% & 55.5\% & 30.0\% & 13.5\% & 0.45 \\
\hline
\end{tabular}
\end{center}
\caption{Extension of Table 1 for KITTI setting only. `0\% Occl.'  denotes initial accuracy. Percentages represent the amount of data occluded. Rows marked BB are for the ISO algorithm operating in a black-box setting.}\label{tab:table1}
\end{table}

\end{document}